\documentclass[11pt]{article}

\usepackage[margin=1in]{geometry}
\usepackage{amsmath,amssymb,amsfonts,amsthm}
\usepackage{booktabs}
\usepackage{graphicx}
\usepackage{algorithm}
\usepackage{algpseudocode}
\usepackage{xcolor}
\usepackage{hyperref}
\usepackage[numbers,sort&compress]{natbib}
\usepackage{microtype}
\usepackage{enumitem}

\hypersetup{
    colorlinks=true,
    linkcolor=blue,
    citecolor=blue,
    urlcolor=blue
}

\newtheorem{definition}{Definition}

\newtheorem{proposition}{Proposition}
\newtheorem{remark}{Remark}

\newcommand{\norm}[1]{\left\lVert #1 \right\rVert}
\newcommand{\inner}[2]{\left\langle #1, #2 \right\rangle}

\newcommand{\polar}{\operatorname{polar}}
\newcommand{\srank}{\operatorname{srank}}

\title{Low-Rank Decay for Grokking in Scale-Invariant Transformers:\\
A Spectral-Geometric View}

\author{Mingyu Li\\Beijing Normal University
}

\date{Preliminary Draft}

\begin{document}
\maketitle

\begin{abstract}
Modern Transformer architectures frequently employ normalization mechanisms such as RMSNorm and Query-Key Normalization, making parts of the model approximately scale-invariant with respect to weight magnitudes. In this regime, standard Frobenius-norm weight decay acts purely along the radial direction of the weight space and cannot directly simplify the function represented by the normalized layer. We study grokking in small algorithmic tasks through this lens and propose \emph{Low-Rank Decay} (LRD), a nuclear-norm-like spectral regularizer whose subgradient---the polar factor $UV^\top$---retains a tangential component even in the scale-invariant setting. This distinction has a concrete dynamical consequence: after the model memorizes the training set and task gradients vanish, L2 decay can no longer reshape the weight spectrum, whereas LRD continues to compress singular values in an $\ell_1$-like fashion. On modular arithmetic tasks, we find that LRD induces rapid effective-rank collapse in Query/Key matrices and expands the data-fraction boundary at which delayed generalization (grokking) occurs. We further provide a spectral-geometric interpretation through the ``needle-to-fan'' expansion of the nuclear-norm subdifferential near low-rank strata.
\end{abstract}


\section{Introduction}

Grokking refers to a delayed generalization phenomenon in which a model first memorizes the training set and only much later transitions to near-perfect test accuracy on algorithmic tasks \citep{power2022grokking}. This behavior suggests that training dynamics are not merely fitting labels, but may involve a transition from a high-capacity memorization solution to a lower-complexity algorithmic circuit.

At the same time, modern neural architectures increasingly rely on normalization layers. In Transformers \citep{vaswani2017attention}, RMSNorm \citep{zhang2019rmsnorm} and Query-Key Normalization (QK-Norm) normalize activations or attention queries/keys, making the resulting computation insensitive to scalar rescaling of certain weight matrices. This raises a basic question: \emph{if a layer is scale-invariant, what does standard weight decay actually regularize?}

The usual Frobenius penalty $\norm{W}_F^2$ controls the magnitude of a weight matrix. However, if the function is invariant under $W \mapsto \alpha W$ for $\alpha>0$, shrinking $\norm{W}_F$ does not directly simplify the function. We therefore explore the hypothesis that, for scale-invariant components, an appropriate inductive bias should act on the \emph{spectral shape} of the matrix, rather than only its scalar norm.

This draft investigates a simple method, \emph{Low-Rank Decay} (LRD), which applies a nuclear-norm-like decay to selected weight matrices. The nuclear norm $\norm{W}_* = \sum_i \sigma_i(W)$ acts as a convex surrogate for rank \citep{candes2009exact} and encourages spectral sparsity. To avoid per-step singular value decomposition (SVD), we approximate the polar factor $UV^\top$ of $W=U\Sigma V^\top$ using Newton--Schulz iteration and apply a decoupled decay step.

Our preliminary experiments on modular addition suggest that LRD induces rapid effective-rank collapse in attention Query/Key matrices and improves grokking under low data fractions. The main contributions of this draft are:

\begin{enumerate}[leftmargin=*]
    \item We formalize the mismatch between Frobenius-norm weight decay and scale-invariant Transformer components, showing that L2 decay acts purely radially and becomes ineffective once task gradients vanish.
    \item We introduce a decoupled Low-Rank Decay update based on a Newton--Schulz approximation of the nuclear-norm subgradient, with adaptive scaling to prevent catastrophic amplification at small weight norms.
    \item We present experimental evidence that LRD induces rank collapse and accelerates grokking, particularly on tasks where L2 fails due to gradient vanishing after memorization.
    \item We construct phase diagrams showing that LRD expands the data-fraction boundary for successful generalization compared to L2.
    \item We provide a spectral-geometric interpretation of LRD near low-rank strata, including the ``needle-to-fan'' expansion of the nuclear-norm subdifferential.
\end{enumerate}

\section{Background and Related Work}

\paragraph{Grokking.}
Grokking was introduced by \citet{power2022grokking} on algorithmic datasets such as modular arithmetic. Later work connected grokking to mechanistic interpretability, showing that generalization can correspond to the formation of structured circuits, for example Fourier or trigonometric circuits in modular addition \citep{nanda2023progress}. \citet{liu2022omnigrok} further extended grokking beyond algorithmic datasets to real-world tasks including image classification and sentiment analysis, indicating that this is a general property of overparameterized optimization rather than a task-specific artifact. Our work focuses on an optimization-side intervention that may accelerate or stabilize the transition toward such low-complexity solutions.

\paragraph{Spectral dynamics and low-rank bias.}
\citet{yunis2024grokking} found that the timing of generalization transitions in grokking closely tracks the emergence of low-rank solutions. In a broader study, \citet{yunis2024spectral} further demonstrated that this low-rank bias is a pervasive phenomenon in deep learning optimization, and that weight decay reinforces it beyond its role as a norm regularizer.

\paragraph{Scale-invariant optimization.}
A parameter block is scale-invariant if multiplying it by a positive scalar does not change the represented function. Prior work on normalization and scale-invariant parameters has shown that weight decay can alter effective learning rates by shrinking the parameter norm, even when it does not directly impose a functional complexity penalty \citep{arora2018optimization,li2020exponential}. We use this as motivation to consider spectral regularization instead of purely radial norm regularization.

\paragraph{Spectral regularization and low-rank bias.}
The singular spectrum of a weight matrix is a natural complexity measure. Nuclear-norm penalties encourage low-rank solutions and are widely used as convex surrogates for rank \citep{candes2009exact}. In deep learning, low-rank adaptation methods such as LoRA impose a hard architectural rank constraint \citep{hu2021lora}. In contrast, LRD is a soft regularizer: it encourages rank reduction but allows task gradients to maintain or regrow singular directions when needed.

\section{Scale Invariance and the Limitation of Weight Decay}

\begin{definition}[Scale invariance]
A function $f(x; W)$ is scale-invariant with respect to $W$ if
\begin{equation}
    f(x; \alpha W) = f(x; W), \qquad \forall \alpha>0.
\end{equation}
\end{definition}

For an attention head with QK-Norm, attention logits take the schematic form
\begin{equation}
    A(X) = \operatorname{softmax}\left(\frac{\operatorname{Norm}(XW_Q)\operatorname{Norm}(XW_K)^\top}{\sqrt{d_h}}\right),
\end{equation}
where $\operatorname{Norm}(z)$ removes the scale of $z$. Hence scaling $W_Q$ or $W_K$ by a positive scalar does not change the normalized queries or keys.

Let $W=\rho\Theta$, where $\rho=\norm{W}_F$ and $\norm{\Theta}_F=1$. Suppose the task loss $\mathcal{L}_{\mathrm{task}}(W)$ depends only on $\Theta$. Consider the continuous-time objective
\begin{equation}
    J(W) = \mathcal{L}_{\mathrm{task}}(W) + \frac{\lambda}{2}\norm{W}_F^2.
\end{equation}
The gradient flow is
\begin{equation}
    \dot W = -\nabla_W \mathcal{L}_{\mathrm{task}}(W) - \lambda W.
\end{equation}
Since $\mathcal{L}_{\mathrm{task}}$ is homogeneous of degree zero with respect to $W$, Euler's homogeneous function theorem gives
\begin{equation}
    \inner{\nabla_W \mathcal{L}_{\mathrm{task}}(W)}{W}=0.
\end{equation}
Thus the task gradient is tangent to the sphere of constant Frobenius norm, while the weight-decay term is radial.

\begin{proposition}[Radial decoupling of Frobenius weight decay]
Under the idealized assumptions above, Frobenius-norm weight decay affects only the radial scale $\rho$ and has no direct projection onto the angular dynamics of $\Theta$.
\end{proposition}

\begin{proof}[Sketch]
The radial dynamics satisfy
\begin{equation}
    \dot \rho = \inner{\dot W}{\Theta}
    = -\inner{\nabla_W \mathcal{L}_{\mathrm{task}}}{\Theta} - \lambda \inner{W}{\Theta}
    = -\lambda \rho.
\end{equation}
For the angular variable, the projection of $-\lambda W=-\lambda\rho\Theta$ onto the tangent space at $\Theta$ vanishes. Hence the decay term changes $\rho$ but not the angular direction directly.
\end{proof}

\begin{remark}
This is an idealized continuous-time statement. In practical training with AdamW, finite steps, non-normalized parameters, and approximate scale invariance, weight decay can still influence optimization indirectly. The claim is not that weight decay is useless, but that it is not a direct functional complexity prior for perfectly scale-invariant parameter blocks.
\end{remark}

\section{Low-Rank Decay}

\subsection{Nuclear-norm objective}

We replace the radial Frobenius penalty with a spectral penalty:
\begin{equation}
    J_{\mathrm{LRD}}(W) = \mathcal{L}_{\mathrm{task}}(W) + \lambda \norm{W}_*,
    \qquad \norm{W}_* = \sum_i \sigma_i(W).
\end{equation}
If $W=U\Sigma V^\top$ is full-rank and has nonzero singular values, one subgradient of $\norm{W}_*$ is
\begin{equation}
    \nabla_W \norm{W}_* = UV^\top = \polar(W).
\end{equation}
This polar factor preserves the singular-vector geometry while discarding singular magnitudes. A decay step along $UV^\top$ reduces singular values in an $\ell_1$-like manner, encouraging small singular values to vanish.

\subsection{Newton--Schulz approximation}

Exact SVD is expensive inside every optimizer step. We therefore approximate the polar factor using Newton--Schulz iteration:
\begin{equation}
    X_0 = \frac{W}{\norm{W}_F + \epsilon},
    \qquad
    X_{k+1} = \frac{1}{2} X_k(3I - X_k^\top X_k).
\end{equation}
For rectangular matrices, one can use the corresponding left- or right-oriented variant depending on matrix shape. In practice, a small fixed number of iterations is used.

\begin{algorithm}[t]
\caption{Decoupled Low-Rank Decay via Newton--Schulz}
\label{alg:lrd}
\begin{algorithmic}[1]
\Require Weight matrix $W$, learning rate $\eta$, decay coefficient $\lambda$, iterations $K$, small $\epsilon>0$
\State $X \gets W/(\norm{W}_F+\epsilon)$
\For{$k=1$ to $K$}
    \State $A \gets X^\top X$
    \State $X \gets \frac{1}{2}X(3I-A)$
\EndFor
\State $G_{\mathrm{reg}} \gets \lambda X$
\State $W \gets W - \eta G_{\mathrm{reg}}$
\end{algorithmic}
\end{algorithm}

\subsection{Subdifferential geometry near low-rank strata}

The nuclear norm becomes non-smooth at rank-deficient matrices. Let $W=U_r\Sigma_r V_r^\top$ be a compact SVD with rank $r$. The subdifferential is
\begin{equation}
\label{eq:subdiff}
    \partial \norm{W}_*
    = \left\{U_rV_r^\top + Z:\ U_r^\top Z=0,\ ZV_r=0,\ \norm{Z}_2\leq 1\right\}.
\end{equation}
When $W$ is square and full-rank, the orthogonal complements are trivial and the subgradient is unique. As rank decreases, the zero-singular-value subspaces enlarge, and the admissible $Z$ directions expand. Geometrically, the subgradient changes from a near-unique ``needle'' direction to a fan-like set over the null singular subspaces.

This suggests the following interpretation of LRD. Early in training, when weights are effectively full-rank, LRD provides a clear polar-factor shrinkage direction. Later, near low-rank strata, the regularizer becomes less prescriptive: it has already removed many weak spectral directions, while task gradients can decide which remaining singular subspaces should be preserved or regrown. This gives LRD a soft-selection behavior rather than a hard fixed-rank constraint.

\section{Experimental Setup}

\paragraph{Task.}
We use modular addition:
\begin{equation}
    c = (a+b) \bmod p,
\end{equation}
with $p=97$ unless otherwise stated. The full dataset consists of all ordered pairs $(a,b)$.

\paragraph{Model.}
The model is a 2-layer Transformer with model dimension $d_{\mathrm{model}}=128$ and 4 attention heads. Pre-RMSNorm and QK-Norm are enabled to emphasize scale-invariant attention components.

\paragraph{Baselines.}
We compare three regularization strategies, all applied in a \emph{decoupled} fashion (after the optimizer step, not through the loss):
\begin{itemize}[leftmargin=*]
    \item \textbf{L2 (Frobenius weight decay):} $W \leftarrow W \cdot (1 - \eta\lambda)$, applied to all non-bias, non-norm parameters.
    \item \textbf{LRD Decoupled:} $W \leftarrow W - \eta\lambda \cdot \frac{\norm{W}_F}{\norm{\polar(W)}_F + \epsilon} \cdot \polar(W)$, applied only to attention and MLP weight matrices within transformer blocks. The adaptive scaling factor $\norm{W}_F / \norm{\polar(W)}_F$ prevents catastrophic amplification when the weight norm is small.
    \item \textbf{LRD Elastic:} Both LRD and L2 applied jointly with independent coefficients $\lambda_{\mathrm{lrd}}$ and $\lambda_{l2}$.
\end{itemize}

\paragraph{Metrics.}
We track train accuracy, test accuracy, loss, weight norms, gradient norms, singular values, and stable rank. For a matrix $W$, stable rank is defined as
\begin{equation}
    \srank(W) = \frac{\norm{W}_F^2}{\norm{W}_2^2}.
\end{equation}
We report stable rank for Query/Key matrices during training.

\paragraph{Grid.}
We sweep training data fraction $f \in \{0.3, 0.35, 0.4, 0.45, 0.5, 0.55, 0.6, 0.65, 0.7, 0.75, 0.8\}$ and regularization strength $\lambda \in \{10^{-4}, 5\times10^{-4}, 10^{-3}, 5\times10^{-3}, 10^{-2}, 5\times10^{-2}\}$ to construct phase diagrams of generalization. A run is counted as successful grokking if it reaches test accuracy above $95\%$.

\section{Results}

\subsection{Phase diagram of generalization}

Figure~\ref{fig:phase} shows the phase diagram for modular addition across training data fractions and decay strengths. L2 weight decay achieves high test accuracy only in the upper-right region (large data fraction and strong decay), while spectral decay (LRD) expands the successful grokking region substantially, particularly at moderate data fractions ($f=0.4$--$0.6$) and intermediate decay strengths ($\lambda = 5\times10^{-4}$ to $5\times10^{-3}$). This expansion is consistent with the hypothesis that nuclear-norm regularization provides a more effective inductive bias than Frobenius-norm regularization in scale-invariant networks.

\begin{figure}[t]
    \centering
    \includegraphics[width=0.95\linewidth]{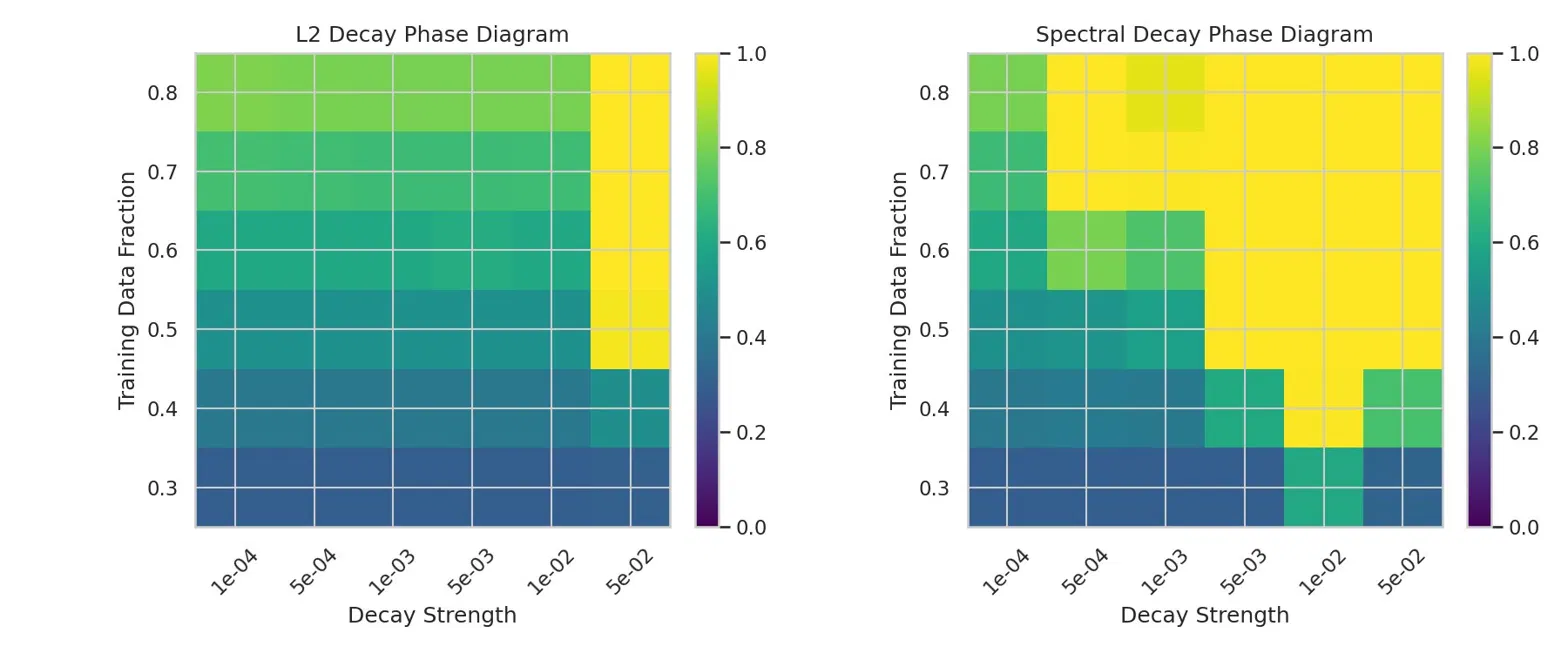}
    \caption{Phase diagram of modular-addition generalization. Left: L2 decay. Right: Spectral (LRD) decay. Color indicates final test accuracy after fixed training steps. LRD expands the grokking boundary to lower data fractions and a wider range of decay strengths.}
    \label{fig:phase}
\end{figure}

\subsection{Grokking speed and rank collapse dynamics}

Figure~\ref{fig:rank} compares grokking speed, effective rank, and singular value spectra for L2 and LRD on modular addition ($f=0.4$). LRD achieves test accuracy $>0.95$ within 3000 steps, while L2 remains below 0.5 over the same period. The effective rank of Query/Key matrices under LRD drops sharply from $\sim$35 to $\sim$3 within the first 500 steps, whereas L2 maintains a rank of $\sim$25 throughout.

The singular value spectrum (normalized by $\sigma_1$) reveals a qualitative difference: under LRD, singular values beyond index $\sim$10 drop by 4--5 orders of magnitude, indicating strong spectral sparsity. Under L2, the spectrum decays more gradually, with tail singular values remaining within 1--2 orders of $\sigma_1$.

The bottom row of Figure~\ref{fig:rank} visualizes the QK attention patterns ($|W_Q W_K^\top|$). L2 produces a noisy, unstructured pattern, while LRD yields a visibly structured pattern with clear block-diagonal features, consistent with the model learning an algorithmic circuit.

\begin{figure}[t]
    \centering
    \includegraphics[width=0.95\linewidth]{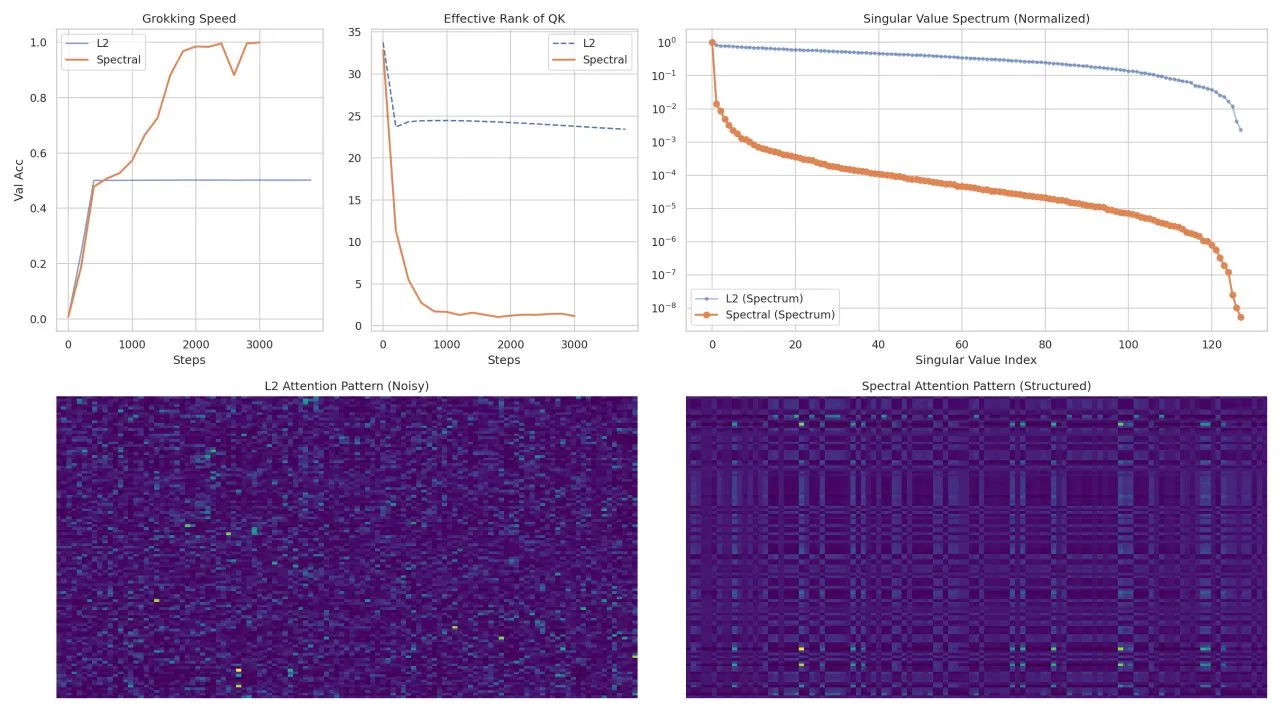}
    \caption{Top left: Grokking speed (test accuracy vs.\ training steps). Top center: Effective rank (stable rank) of QK matrices. Top right: Normalized singular value spectrum at end of training (log scale). Bottom: QK attention patterns for L2 (left) and LRD (right).}
    \label{fig:rank}
\end{figure}

\subsection{The gradient-vanishing regime: why L2 stalls}

A critical observation is the behavior after memorization. Once the model achieves perfect training accuracy, the task loss reaches zero and gradient norms vanish ($\norm{\nabla \mathcal{L}_{\mathrm{task}}}\approx 0$). In this regime:

\begin{itemize}[leftmargin=*]
    \item \textbf{L2 decay} continues to apply $W \leftarrow W(1-\eta\lambda)$, which shrinks $\norm{W}_F$ exponentially but does not change the angular direction $\Theta = W/\norm{W}_F$. Since the network is scale-invariant, this has no functional effect. The model remains trapped in the memorization solution.
    \item \textbf{LRD} continues to apply $W \leftarrow W - \eta\lambda\cdot\polar(W)$. The polar factor $UV^\top$ is independent of the scale of $W$ and provides a directional update: it subtracts a constant from each singular value, preferentially eliminating small ones. This reshapes the singular spectrum even without task gradients, eventually driving the weight into a low-rank configuration that supports generalization.
\end{itemize}

This asymmetry---L2 acts radially and becomes inert in scale-invariant networks after gradient vanishing, while LRD retains a tangential spectral-shaping component---is the central mechanistic distinction.

Table~\ref{tab:comparison} summarizes the key differences.

\begin{table}[t]
\centering
\caption{Comparison of L2 and LRD regularization in the scale-invariant setting.}
\label{tab:comparison}
\begin{tabular}{@{}lll@{}}
\toprule
\textbf{Property} & \textbf{L2 Weight Decay} & \textbf{Low-Rank Decay (LRD)} \\
\midrule
Penalty & $\norm{W}_F^2 = \sum_i \sigma_i^2$ & $\norm{W}_* = \sum_i \sigma_i$ \\
Effect on $\sigma_i$ & $\sigma_i \leftarrow \sigma_i(1-\eta\lambda)$ (multiplicative) & $\sigma_i \leftarrow \sigma_i - \eta\lambda$ (subtractive) \\
Stable rank & Unchanged & Decreases \\
Gradient direction & $\nabla\norm{W}_F^2 = 2W$ (radial) & $\nabla\norm{W}_* = UV^\top$ (tangential component) \\
After $\nabla\mathcal{L}=0$ & Inert (scale-invariant) & Active (spectral reshaping) \\
\bottomrule
\end{tabular}
\end{table}

\section{Discussion}

\subsection{LRD as tangential spectral pressure}

The key insight of this work is not merely that LRD encourages low rank, but \emph{why} it remains effective in a regime where L2 fails. In scale-invariant networks, the loss landscape has a degenerate radial direction: any rescaling $W\mapsto\alpha W$ leaves the loss unchanged. L2 weight decay acts along this degenerate direction and thus cannot exert functional regularization pressure after gradient vanishing. LRD, by contrast, acts through the polar factor $UV^\top$, which has a nontrivial tangential component that reshapes the singular spectrum independently of the weight scale.

This distinction becomes operationally important precisely when the model has memorized the training set. At that point, task gradients vanish and the only forces acting on the weights are the regularizers. L2 continues to shrink $\norm{W}$ without functional consequence; LRD continues to compress the singular spectrum, eventually driving the weights into a low-rank configuration that aligns with an algorithmic (generalizing) solution.

\subsection{Connection to spectral dynamics literature}

Our findings are consistent with the spectral dynamics framework of \citet{yunis2024spectral}, who observed that weight decay strengthens an implicit low-rank bias across diverse architectures. Our contribution is complementary: we show that \emph{explicit} spectral regularization (LRD) can accelerate and extend this bias, particularly in settings where implicit bias alone is insufficient (low data fractions, scale-invariant architectures).

Recently, \citet{xu2026lowdim} proposed a geometric account of grokking in which weight trajectories are confined to a low-dimensional execution manifold, with transverse curvature accumulating until the trajectory escapes into the generalizing solution. LRD's spectral compression may interact with the dimensionality of this execution manifold, a direction worth further exploration.

\subsection{Limitations}

This work has several important limitations. Current evidence is based primarily on modular addition; extension to modular multiplication, permutation composition (S5), and other tasks is needed to establish generality. Multi-seed results with confidence intervals are needed to rule out seed-dependent effects. We have not established a causal link between rank collapse and generalization---temporal precedence alone is insufficient. The interaction between LRD and the optimizer's implicit spectral bias (observed even without regularization) requires further disentanglement. Finally, training instabilities (rank oscillations, gradient spikes) observed in some configurations are not fully understood.

\section{Conclusion}

We presented Low-Rank Decay, a nuclear-norm-like spectral regularizer for grokking in scale-invariant Transformers. The central insight is that, when QK-Norm renders attention layers scale-invariant, Frobenius-norm weight decay acts purely radially and becomes functionally inert after gradient vanishing---while LRD's polar-factor update retains a tangential component that continues to reshape the singular spectrum. This mechanistic distinction manifests empirically: LRD induces rapid rank collapse and accelerates grokking on modular arithmetic, expanding the data-fraction boundary for successful generalization compared to L2 weight decay. Phase diagrams confirm that the grokking region under LRD is substantially larger, particularly at moderate data fractions where L2 fails entirely.

Future work includes extending the evaluation to additional algebraic tasks (modular multiplication, permutation composition), establishing causality through targeted spectral interventions, and exploring connections between LRD-induced low-rank structure and the execution manifold geometry recently identified by \citet{xu2026lowdim}.

\bibliographystyle{plainnat}

\end{document}